\documentclass{article} 
\usepackage{iclr2025_conference,times}

\usepackage{amsmath,amsfonts,bm}

\def\eqref#1{equation~\ref{#1}}

\def\1{\bm{1}}

\def\vv{{\bm{v}}}
\def\vw{{\bm{w}}}
\def\vx{{\bm{x}}}

\DeclareMathAlphabet{\mathsfit}{\encodingdefault}{\sfdefault}{m}{sl}
\SetMathAlphabet{\mathsfit}{bold}{\encodingdefault}{\sfdefault}{bx}{n}

\newcommand{\R}{\mathbb{R}}

\DeclareMathOperator*{\argmax}{arg\,max}

\usepackage[utf8]{inputenc}
\usepackage{graphicx}
\usepackage{bbm}
\usepackage{array}
\usepackage{caption}
\usepackage{float}
\usepackage{subcaption}
\usepackage{subfiles}
\usepackage{tikz-cd, tikz, tikzsymbols}
\usepackage{pgfplots}
\usepackage{enumerate}
\usepackage{mathrsfs}
\usepackage{multicol}
\usepackage{setspace}
\usepackage{hyperref}
\usepackage{mwe}
\pgfplotsset{width=10cm,compat=1.9}

\usepackage{amsmath, amsfonts, mathtools, amsthm, amssymb}
\usepackage{physics, bm, cancel}
\usepackage{arydshln}

\usepackage{wrapfig} 
\usepackage{booktabs}

\usepackage{algorithm}
\usepackage{algpseudocode}
\newcommand{\loss} {\mathcal{L}}
\renewcommand{\grad}{\nabla}

\newcommand{\eat}[1]{}

\newtheorem{theorem}{Theorem}[section]
\newtheorem{proposition}[theorem]{Proposition}

\newtheorem{lemma}[theorem]{Lemma}

\theoremstyle{plain}

\newtheorem{manualpropositioninner}{Proposition}
\newenvironment{manualproposition}[1]{%
  \renewcommand\themanualpropositioninner{#1}%
  \manualpropositioninner
}{\endmanualpropositioninner}

\theoremstyle{definition}

\theoremstyle{remark}

\usepackage{hyperref}
\usepackage{url}

\title{Improving Learning to Optimize Using Parameter Symmetries}

\author{Guy Zamir\thanks{Work done while a student at the University of California, San Diego.} \\
University of Wisconsin-Madison\\
\texttt{gzamir@wisc.edu} \\
\And
Aryan Dokania \\
University of California, San Diego \\
\texttt{adokania@ucsd.edu} \\
\AND
Bo Zhao \\
University of California, San Diego \\
\texttt{bozhao@ucsd.edu} \\
\And
\hspace{85pt} Rose Yu \\
\hspace{85pt} University of California, San Diego \\
\hspace{85pt} \texttt{roseyu@ucsd.edu} \\
}

\iclrfinalcopy 
\begin{document}

\maketitle

\begin{abstract}
We analyze a learning-to-optimize (L2O) algorithm that exploits parameter space symmetry to enhance optimization efficiency. Prior work has shown that jointly learning symmetry transformations and local updates improves meta-optimizer performance \citep{zhao2023improving}. Supporting this, our theoretical analysis demonstrates that even without identifying the optimal group element, the method locally resembles Newton’s method. We further provide an example where the algorithm provably learns the correct symmetry transformation during training. To empirically evaluate L2O with teleportation, we introduce a benchmark, analyze its success and failure cases, and show that enhancements like momentum further improve performance. Our results highlight the potential of leveraging neural network parameter space symmetry to advance meta-optimization.
\end{abstract}

\section{Introduction}
Optimization is fundamental to machine learning. While traditional hand-crafted optimizers such as SGD and Adam have proven effective across diverse tasks, recent advances in learning-to-optimize (L2O) have demonstrated the potential to automatically discover adaptive optimizers that outperform manually designed ones \citep{andrychowicz2016learning, finn2017model}. 
By parameterizing the optimization process and learning from a population of models, L2O methods derive update rules that generalize across multiple problem instances. 
Despite this progress, existing L2O methods primarily learn local update rules and often overlook global structural properties of the parameter space. 


Recent work explores incorporating parameter space symmetry into L2O, leveraging transformations that move parameters within symmetry orbits before applying updates \citep{zhao2023improving}.
This builds on the observation that different points in an orbit can exhibit significantly different gradients and learning dynamics \citep{kunin2021neural}, suggesting that moving across orbits before updating could lead to faster convergence and better generalization. 
By integrating this notion of teleportation into L2O, this method provides a new perspective on optimization as a structured exploration of parameter space rather than a purely local search process.

To support this approach, we provide theoretical insights showing that teleportation can locally resemble Newton’s method. 
Furthermore, we prove that under certain conditions, L2O can learn optimal teleportation strategies that improve convergence rates. 
Additionally, we introduce a benchmark for evaluating L2O with teleportation, analyze its strengths and limitations, and demonstrate practical improvements on classification models with larger parameter spaces.  
By bridging symmetry-aware optimization and meta-learning, our results highlight the potential of leveraging symmetry—a fundamental property of neural network parameter spaces—to enhance meta-optimization.
Learning how to teleport also allows one to leverage the benefits of teleportation without performing the potentially expensive optimization on group orbits.

\section{Related Work}
\paragraph{Learning to Optimize (L2O).}

Optimization-based meta-learning, also known as learning to optimize, aims to improve optimization strategies by training models to learn task-specific optimizers. The pioneering work of \citet{andrychowicz2016learning} proposed LSTMs as meta-optimizers, inspiring subsequent research on gradient-based meta-learning. \citet{finn2017model} introduced Model-Agnostic Meta-Learning (MAML), which adapts quickly to new tasks via second-order gradients, while \citet{nichol2018first} proposed Reptile, a first-order variant reducing computational overhead. Extensions include implicit gradient-based methods like iMAML \citep{rajeswaran2019meta} and Meta-SGD \citep{li2017meta}, which refine adaptation efficiency. Recent advances focus on improving scalability and generalization, such as learned optimizers \citep{metz2019understanding, metz2022practical}, learned step-size schedules \citep{wichrowska2017learned}, and curriculum-based meta-training \citep{zhou2023curriculum}. 
\citet{zhao2023improving} propose integrating parameter symmetry in L2O. We provide additional theoretical justification on the effectiveness of incorporating symmetry and improve their algorithm by including a momentum term.

L2O is also closely connected to recent work in weight space learning, where neural network weights are treated as data objects for tasks such as model analysis and generation \citep{schurholt2025neural}. 
These approaches learn from distributions over trained neural network weights to model network populations \citep{kofinas2024graph, lim2024graph, zhou2023permutation, vo2024equivariant, zhou2024universal, kalogeropoulos2024scale}.
Similarly, L2O methods operate over distributions of optimization tasks and aim to generalize across diverse initializations. 
In weight space learning, parameter symmetries often play a key role in building equivariant meta-models. 
In this paper, we explore another use of symmetry, by enabling an additional degree of freedom in the learned optimizers.

\paragraph{Parameter symmetry in optimization.}
Symmetry in neural network parameter spaces induces variations in gradients and learning dynamics along different points in the same orbit \citep{van2017l2, tanaka2021noether}. 
This phenomenon has been exploited to accelerate training by searching for more favorable points in the orbit, for instance by minimizing parameter norms \citep{stock2019equi, saul2023weight}, maximizing gradient magnitude \citep{zhao2022symmetry, zhao2023improving}, or via randomly sampling \citep{armenta2023neural}. 
An alternative line of work makes optimization itself invariant to parameter symmetries, ensuring consistent convergence across orbit members. 
Examples include rescaling-invariant methods that exploit path products \citep{neyshabur2015path-sgd, meng2019mathcal} or constrain weights on a manifold \citep{badrinarayanan2015symmetry, huang2020projection, yi2022accelerating}, as well as broader geometric approaches such as natural gradient descent \citep{amari1998natural, song2018accelerating, kristiadi2023geometry}. 
Motivated by the success of incorporating symmetry in optimization, we explore the integration of symmetry into learning-to-optimize frameworks.

\section{Background}
\label{section:background}

In this section, we provide background on leveraging parameter space symmetry in optimization and learning to optimize.

\subsection{Teleportation and Symmetry}
Teleportation \citep{armenta2023neural, zhao2022symmetry} is an optimization strategy that aims to enhance first-order methods by exploiting symmetry of the optimizee $f$.

Suppose we have some optimizee $f:\mathbb{R}^{n} \to \mathbb{R}$ we would like to minimize over $\mathbb{R}^{n}$.
When we assume that $f$ is  $L$-smooth for some  $L<\infty$ and update using gradient descent with step size  $\alpha\in (0,\frac{1}{L}]$, the descent lemma holds \citep{garrigos2023handbook}: \[
	f(\mathbf{x}_{t+1}) \leq f(\mathbf{x}_t) - \frac{\alpha}{2}\| \nabla f(\mathbf{x}_t) \|_2^2
.\]
This fact suggests that replacing $\mathbf{x}_t$ with
\[ 
\mathbf{x}_t '\in \argmax_{\mathbf{x}\in \R^n,  f(\mathbf{x}) \leq f(\mathbf{x_t})} \|\nabla f(\mathbf{x}) \|
\]
will accelerate optimization for at least one subsequent step.

Teleportation leverages the above as follows.
On step \( t \in K \), we select a transformation \( g: \mathbb{R}^{n}\to\mathbb{R}^{n} \) such that \( f(g(\mathbf{x}_t)) \leq f(\mathbf{x}_t) \) and the gradient norm \(\|\nabla f(g(\mathbf{x}_t))\|\) is maximized.
We then set $\mathbf{x}_t' = g(\mathbf{x}_t)$ and proceed with our optimization algorithm of choice.

We often consider the case where the set of teleportation parameters $G$ is a group of symmetries under which the objective $f$ is invariant.
That is, \[
f(g(\mathbf{x})) = f(\mathbf{x}) \, \quad \forall g\in G \, \quad \forall \mathbf{x}\in \mathbb{R}^{n}
.\]
In this case, finding the teleportation parameter amounts to solving the problem \[
g \in  \argmax_{g\in G} \|\nabla f(g(\mathbf{x}))\|
.\] 

\subsection{Learning to Optimize}

The goal of learning to optimize (L2O) is to obtain an algorithm that can efficiently minimize some function $f$ starting at the point $x_0$, with $(f,x_0)$ being drawn from some underlying distribution of tasks $\mathcal{T}$.
To do so, we learn an L2O optimizer $m_\phi : \mathcal{Z} \to \mathbb{R}^n$
and optimize by using the update rule
		$$\mathbf{x}_{t+1} \gets \mathbf{x}_t + m_\phi(\mathbf{z}_t),$$
where $\mathbf{z}_t$ is in the input feature space $\mathcal{Z}$, and \(\phi\) are trainable parameters in $m$.
The L2O optimizer $m$ is usually a neural network (in particular, an LSTM) parametrized by learnable weights $\phi$,
and the input feature $\mathbf{z}_t$ may include information about gradients and iterates.

To learn $\phi$, we need a training loss. A commonly used loss is
$$\mathcal{L}(m) = \sum_{t=1}^T w_t f(\mathbf{x}_t)$$
for some re-weighting coefficients $w_1,\dots,w_T$.
We train according to the following algorithm.

\begin{algorithm}
\caption{L2O Training \cite{l2o}}\label{alg:cap}
\begin{algorithmic}
\Require Number of training runs $R$, training run epochs $N$, unroll frequency $T$, untrained optimizer $m$ parametrized by $\phi$, distribution of tasks $\mathcal{T}$, weights $w_t$
\For{$r = 1,\dots,R$}
    \State Sample optimizee and initialization $(f,x_0) \sim \mathcal{T}$
    \State $\mathcal{L} \gets 0$

    \For{$t=1,\dots,N$}
    \State $\mathbf{z}_{t} \gets \nabla f(\mathbf{x}_{t-1})$
    \State $\mathbf{x}_t \gets \mathbf{x}_{t-1} + m_\phi(\mathbf{z}_{t})$
	\State $\mathcal{L} \gets \mathcal{L} + w_t f(\mathbf{x}_t)$

	\If{$t \,\%\, T == 0$}
	    \State Perform a gradient-based update on \(\phi\) using \(\nabla_\phi \mathcal{L}\).
	    \State $\mathcal{L} \gets 0$
	\EndIf

    \EndFor

\EndFor
\end{algorithmic}
\end{algorithm}

Now, there are various reasons one might use an L2O optimizer.
First, L2O optimizers can learn strategies beyond vanilla gradient descent and in some cases outperform classical optimizers on certain task distributions.
\cite{goog} observed that trained L2O optimizers implicitly develop behaviors analogous to momentum, gradient clipping, and adaptive learning rate schedules. 
Perhaps more ambitiously, since neural networks are universal approximators in theory, L2O methods have the potential to discover optimization strategies that surpass traditional handcrafted algorithms \cite{l2o}.
In particular, L2O optimizers may be able to capture and perform updates based on second-order information.
While strong empirical evidence supporting this idea has yet to emerge, there is promise if we provide sufficient structure to the model, as in \cite{lodo}.

In this paper, we aim to demonstrate that an L2O optimizer augmented with symmetry-based teleportation \citep{zhao2023improving} can learn more effective optimization steps, potentially enabling it to outperform classical methods by adaptively navigating optimization landscapes.
We consider a modified version of Algorithm \ref{alg:cap} that additionally learns to teleport. 
Specifically, we augment the update to incorporate a learned group element $g$: $\mathbf{x}_t \gets g \cdot (\mathbf{x}_{t-1} + m_\phi(\mathbf{z}_{t}))$.
This allows the optimizer not only to perform local gradient-informed updates but also to leverage symmetry transformations.


\section{Theoretical Intuitions}
We provide theoretical intuitions suggesting that learning symmetry transformation in addition to learning local update rules improves L2O in convex settings.
First, we show that Newton's method has a local effect of increasing gradients, similarly to what teleportation does. This means we gain benefits from teleportation even if we do not teleport to the optimal position.
Second, we analyze an example where the L2O optimizer can provably learn the correct teleportation transformation during training.

\subsection{Increase of Gradient Norm in Second-Order Methods}
To provide additional intuition why teleportation speeds up optimization in L2O, we show that it behaves similarly to Newton's method locally. This means that in learning to optimize, additional learning a symmetry transformation can be locally helpful, even if the algorithm does not learn the perfect teleportation.

Consider a twice differentiable convex function $\loss(\vw)$ that has an invertible Hessian. Let $\vv_1 = -\grad \loss$ be the gradient descent update direction, and $\vv_2 = -H^{-1} \grad \loss$ be the update direction in Newton's method. The projection of $\vv_2$ on $\vv_1$ is  given by 
$    \vv_{\parallel} 
    = \frac{\vv_2 \cdot \vv_1}{\|\vv_1\|^2} \vv_1 ,
$
and the component of $\vv_2$ that is orthogonal to $\vv_1$ is 
$
    \vv_{\bot} 
    = \vv_2 - \vv_{\parallel} .
$
%

\begin{wrapfigure}{R}{0.4\columnwidth}
    \begin{center}
        \vskip -20pt
        \includegraphics[width=0.4\columnwidth]{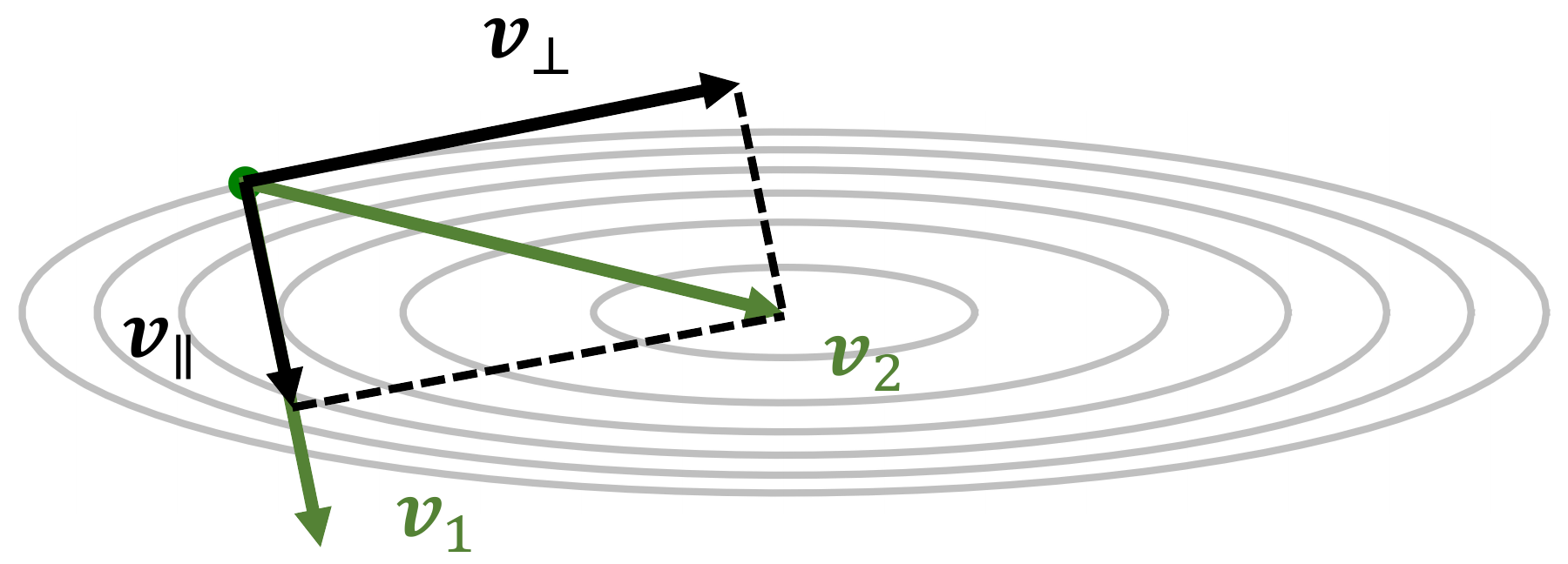}
        \caption{The gradient norm increases along $\vv_{\bot}$, the component of Newton's direction ($\vv_2$) that is orthogonal to the gradient ($\vv_1$).}
        \vskip -20pt
    \label{fig:components}
    \end{center}
\end{wrapfigure}

Teleportation moves parameters in the symmetry direction that increases $\|\grad \loss\|$. The symmetry component of Newton's method has a similar role.
The following proposition states that the gradient norm increases along the symmetry  component for convex functions.

\begin{proposition}
\label{prop:positive-derivative-convex}
For convex function $\loss$, the directional derivative of $\left\| \frac{\partial \loss}{\partial \vw} \right\|_2^2$ along the direction of $\vv_{\bot}$ is non-negative. That is,
\begin{align}
    \vv_{\bot} \cdot \frac{\partial}{\partial \vw} \left\| \frac{\partial \loss}{\partial \vw} \right\|_2^2 \geq 0.
\label{eq:positive-derivative-convex}
\end{align}
\end{proposition}

Figure \ref{fig:components} visualizes the vectors used in the Proposition \ref{prop:positive-derivative-convex} for a convex quadratic function with two parameters. 
While any vector can be decomposed into a gradient component $\vv_\parallel$ and a symmetry component $\vv_\bot$, the gradient norm does not necessarily increase along $\vv_\bot$. 
Therefore, Proposition \ref{prop:positive-derivative-convex} implies that the symmetry component of Newton's method resembles teleportation.

This result also suggests that small teleportation brings the subsequent first-order updates closer to second-order updates, at least for convex functions. This advances the results in \cite{zhao2022symmetry}, which only says that teleportation brings subsequent first-order updates equal to second-order updates at critical points of $\left\| \frac{\partial \loss}{\partial \vw} \right\|_2^2$.
See proofs and further discussions in Appendix \ref{appendix:positive-derivative-convex}.

\subsection{Learnability of Symmetry Transformation}
\label{sec:theta_update}

We next analyze an example where a modified L2O optimizer can learn correct teleportation strategies during training. 
Consider an objective function $f: \mathbb{R}^n \to \mathbb{R}$ that maps parameters to a loss value. 
Let $g: \theta, x \mapsto x$, with $x$ denoting parameters, be a teleportation operation parameterized by $\theta$.
Let $x_0$ and $\theta_0$ be initial parameter and group element, and $\alpha, \beta \in \R$ be learning rates for parameters and group elements. 
We consider Algorithm \ref{alg:lttotf}, a simplified L2O algorithm that learns a single rotation, $\theta$, used in teleportation.

\begin{algorithm}
\caption{Learn to Teleport on the Fly}\label{alg:lttotf}
\begin{algorithmic}
\State {\bfseries Input:} Objective function $f$, teleporting oracle $g$ parameterized by $\theta$, initial $x_0$ and $\theta_0$, learning rates $\alpha$ and $\beta$.
\For{$t = 0, 1, 2, ...$}
    \State $y_t \xleftarrow{} g_{\theta_t}(x_t)$
    \State $x_{t+1} \xleftarrow{} y_t - \alpha \nabla_{y_t} f(y_t)$
    \State $\theta_{t+1} \xleftarrow{} \theta_t - \beta \nabla_{\theta_t} f(x_{t+1})$
\EndFor
\end{algorithmic}
\end{algorithm}

We analyze the case where $f$ is given by $f(x) = x^T A x$ for some positive definite matrix $A \in \mathbb{R}^{2 \times 2}$, i.e., $f$ is a convex quadratic function with a minimum at $(0,0)$. Since $A$ is positive definite, there exists a unique positive definite matrix $B \in \mathbb{R}^{2 \times 2}$ such that $A = B^T B$.
Thus, we can rewrite $f(x) = x^T A x = \| B x \|^2$.

\begin{wrapfigure}{R}{0.4\textwidth}
\vskip -15pt
\centering
\includegraphics[width=0.4\textwidth]{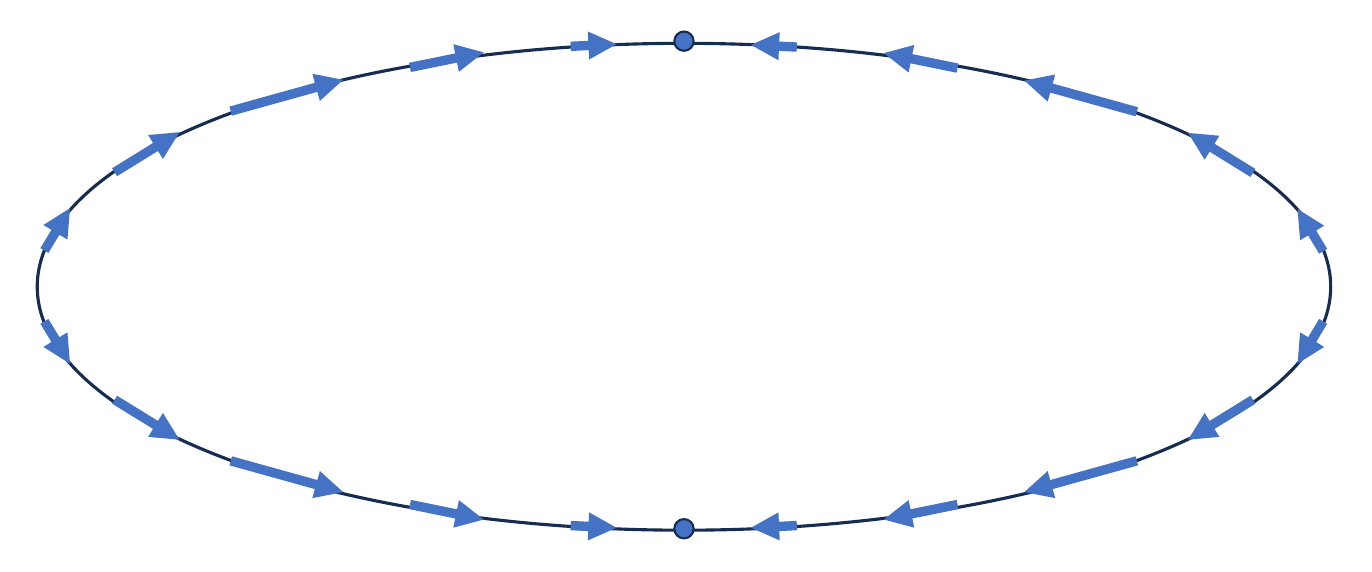}
\caption{Update direction of $\theta$ during learning to teleport.}
\label{fig:theta_update}
\vskip -10pt
\end{wrapfigure}

We parametrize the teleportation operation by a single scalar $\theta \in \mathbb{R}$, which represents a rotation angle.
Let the teleportation transformation be $g_{\theta}(x) = B^{-1} R_\theta B x$.
where $R_\theta$ is the rotation matrix that represents a 2D rotation by angle $\theta$.
It is clear that this transformation preserves the function, i.e. $f(x) = f(g_\theta(x))$ for all $\theta$.

In Appendix \ref{appendix:theta_update}, we derive a closed-form update for \(\theta\). This update consistently rotates parameters toward alignment with the shorter axis of the ellipse formed by the level sets of \(f\), where the gradient is larger (Figure \ref{fig:theta_update}).
Additionally, $\theta$ update fastest when $x$ is halfway between the axes. 
These results suggest that teleportation is a learnable enhancement to L2O, improving optimization by incorporating second-order-like behavior while maintaining computational efficiency.

\vspace{10pt}
\section{Experiments: Learning to Optimize with Teleportation}
\subsection{Test Functions}
While theoretical and empirical evidence suggests that exploiting parameter space symmetry can improve L2O, a systematic evaluation has been lacking. To address this, we introduce a set of low-dimensional functions $f:\mathbb{R}^{2}\to \mathbb{R}$ for an easy teleportation implementation,
faster training for the L2O optimizer,
and the ability to plot and visualize the optimization trajectories\footnote{The code used for our experiments is available at \url{https://github.com/Rose-STL-Lab/L2O-Symmetry}.}. 

\subsubsection{Task distributions}
Consider the family of functions \[
	\mathcal{F} = \Big\{f_h(x,y) = \|h(x,y)\|_2^2 = h(x,y)_1^2 + h(x,y)_2^2 \,\big| \, h:\mathbb{R}^2\to\mathbb{R}^2 \text{ bijective}\Big\} 
.\] 
Each $f_h \in \mathcal{F}$ has $SO(2)$ symmetry.
That is, for any rotation $R_\theta$, we have that \[f_h(x,y) = f_h\Big(h^{-1}\big(R_\theta(h(x,y))\big)\Big).\]

\paragraph{Generalized Booth Functions}

In the first task, we generate a family of ellipses by sampling nonsingular matrix $A\in \mathbb{R}^{2\times 2}$ and vector $b \in \mathbb{R}^2$,
then setting 
\[h(x,y) = A\begin{bmatrix} x \\ y \end{bmatrix} + b.\]
Note that by setting  $A = \begin{bmatrix} 1 & 2 \\ 2 & 1 \end{bmatrix}$ and $b = \begin{bmatrix} -7 \\ -5 \end{bmatrix}$,  $f_h$ becomes the Booth function.
Furthermore, any  $f_h$ generated by this method will be convex, since a convex function composed with an affine function is still convex.

In Figure \ref{fig:ef}, we fix $f_h$ with $h$ defined by taking  $A=\begin{bmatrix} 0.5 & 0 \\ 0 & 3 \end{bmatrix}$ and $b = \begin{bmatrix} 0 \\ 0 \end{bmatrix}$, 
and sample $x_0$ by taking each coordinate IID from a standard normal.
In Figure \ref{fig:ev}, we sample $f_h$ with $h$ defined by $A$ and $b$, where each entry of  $A$ and  $b$ are taken IID from a standard normal. 
We sample  $x_0$ as above.

\paragraph{Generalized Rosenbrock Functions}

In the second task, we generate a family of nonconvex, Rosenbrock-like functions 
\[h(x,y) = \big( c(x) y + d(x), ax + b\big),\]
where $c, d: \R \to \R$, and $a,b\in \mathbb{R}$.
This map is invertible, with inverse 
\[h^{-1}(x,y) = \left(\frac{y-b}{a}, \frac{x-d\left( \frac{y-b}{a} \right) }{c\left( \frac{y-b}{a} \right) }\right)
.\]
Notably, we can recover the Rosenbrock function by taking $f_h$ with $c(x) = -10$,  $d(x) = 10x^2$, $a=1$, and  $b=-1$.
Hence, we can consider this family of functions to be a generalization of the Rosenbrock function.
The functions generated by this method will generally be nonconvex.

In Figure \ref{fig:rf}, we fix $f_h$ with $h$ defined by taking $a=1, b=-1, c(x)=-2,$ and  $d(x)=0.4x^2$.
We sample $x_0$ as above.
In Figure \ref{fig:rv}, we sample $f_h$ with  $h$ defined by $a, b, c(x)=c_1,$ and  $d(x) = d_1x^2 + d_2x + d_3$, and $a,b,c_1,d_1,d_2,d_3$ taken IID from a standard normal.
We sample $x_0$ as above.

\subsubsection{Results}
Although teleportation shows promise in some visualizations, vanilla L2O consistently outperforms the teleportation-augmented variant across both function families (Figures \ref{fig:ef}–\ref{fig:rv}).
One explanation is that teleportation—while beneficial in some settings—may be harmful in others. A learned teleport strategy that improves convergence on one distribution might misalign with another. This aligns with earlier observations that maximizing gradient norm via symmetry may sometimes hinder optimization rather than help.
Thus, these mixed results reflect not just a limitation of L2O, but a broader insight: learning good teleportation strategies is itself a hard problem.

\begin{figure}[H]
    \centering
    \begin{minipage}{0.5\textwidth}
        \centering
        \includegraphics[width=\textwidth]{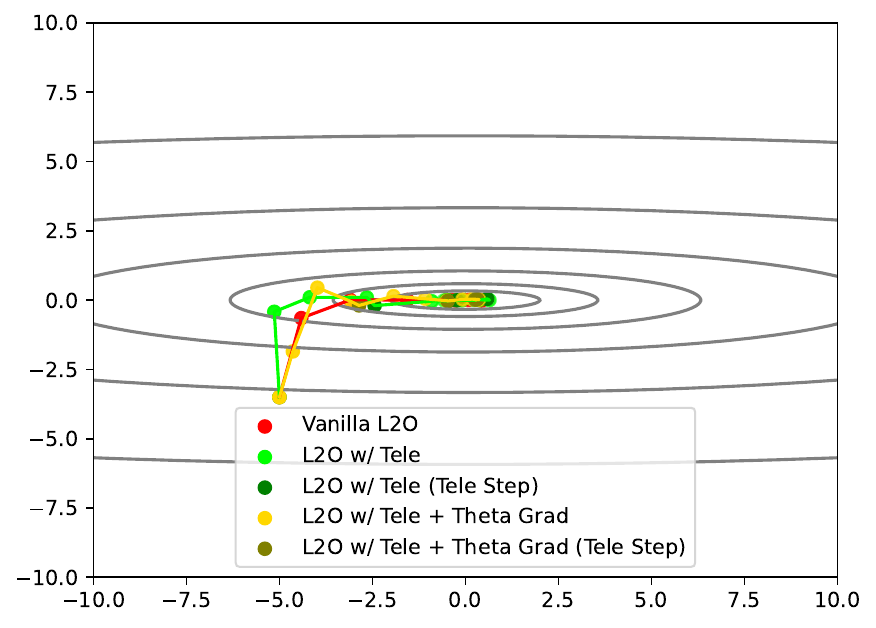}
    \end{minipage}\hfill
    \begin{minipage}{0.5\textwidth}
        \centering
        \includegraphics[width=\textwidth]{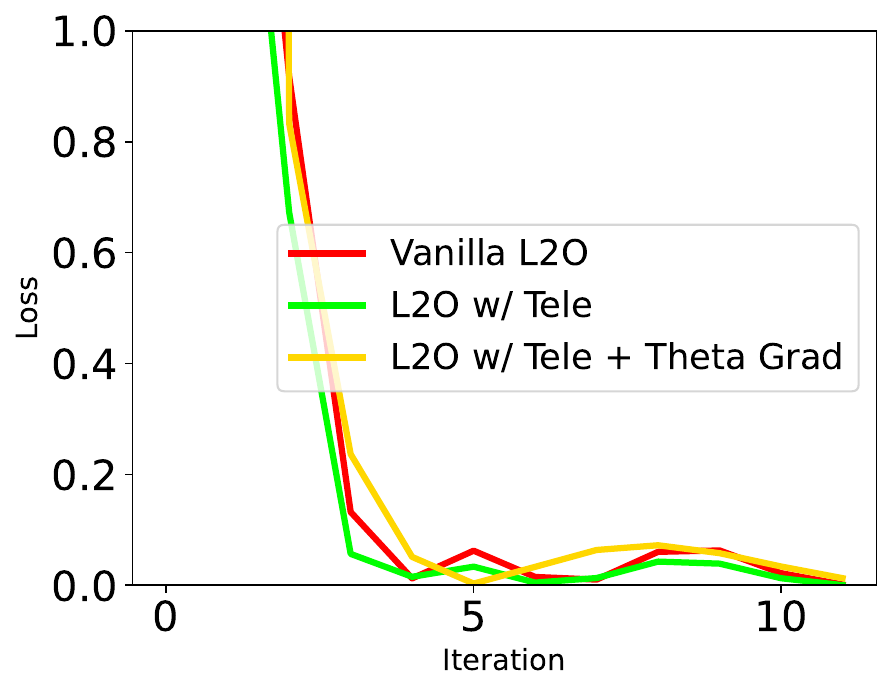}
    \end{minipage}
\caption{Comparison of vanilla L2O with and without learned teleportation for fixed objective ellipse functions.}
\label{fig:ef}
\end{figure}

\begin{figure}[H]
    \centering
    \begin{minipage}{0.5\textwidth}
        \centering
        \includegraphics[width=\textwidth]{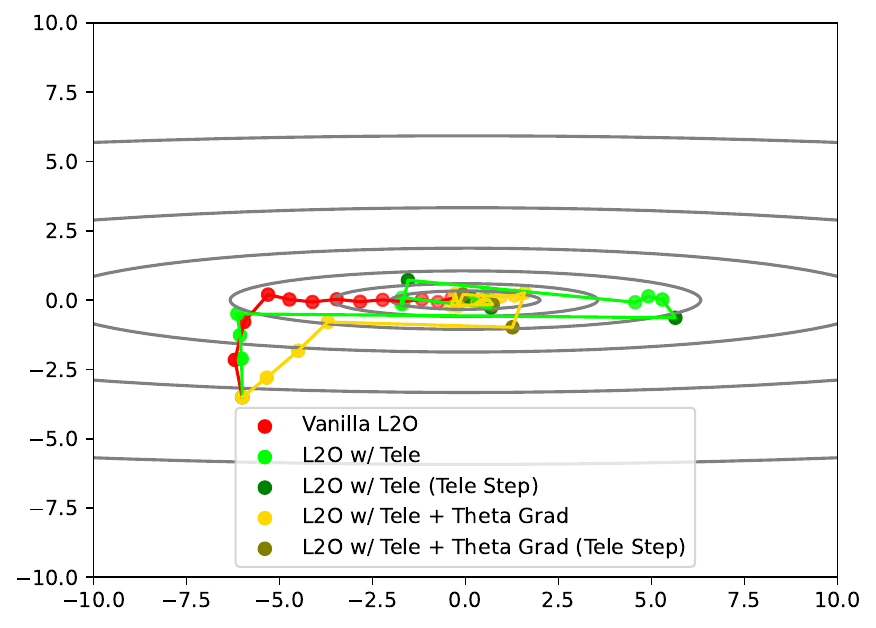}
    \end{minipage}\hfill
    \begin{minipage}{0.5\textwidth}
        \centering
        \includegraphics[width=\textwidth]{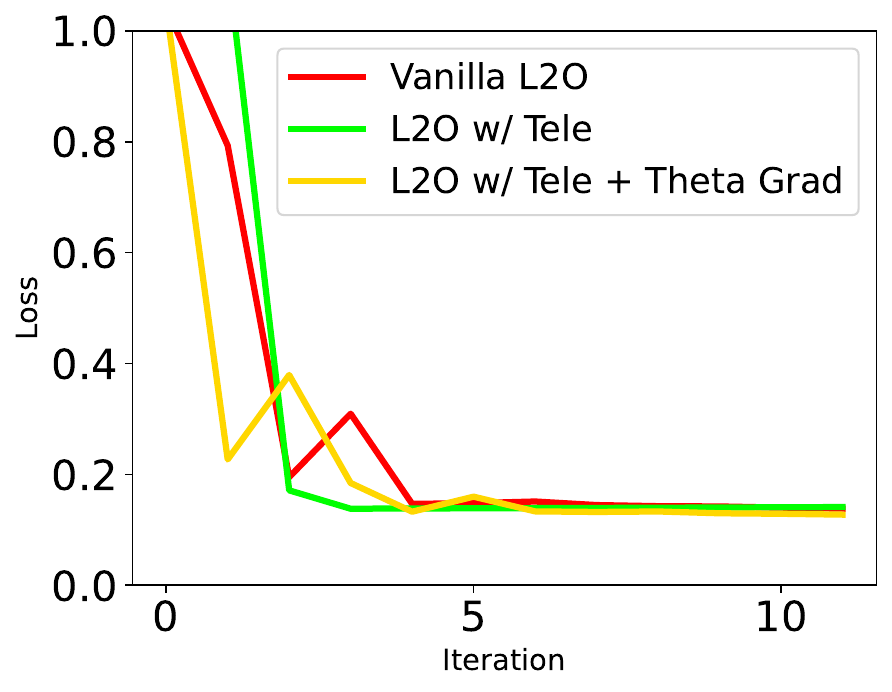}
    \end{minipage}
\caption{Comparison of vanilla L2O with and without learned teleportation for variable objective ellipse functions.}
\label{fig:ev}
\end{figure}

\begin{figure}[H]
    \centering
    \begin{minipage}{0.5\textwidth}
        \centering
        \includegraphics[width=\textwidth]{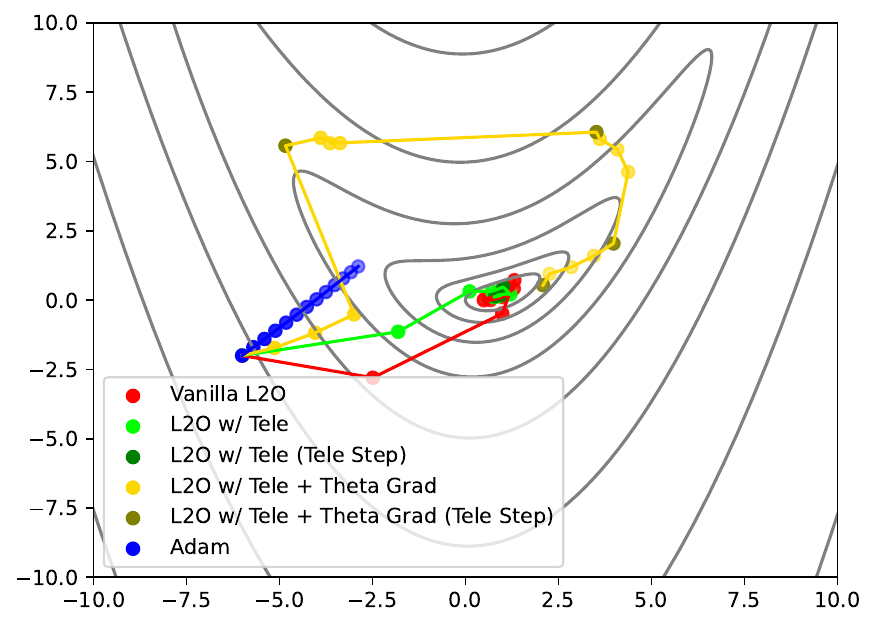}
    \end{minipage}\hfill
    \begin{minipage}{0.5\textwidth}
        \centering
        \includegraphics[width=\textwidth]{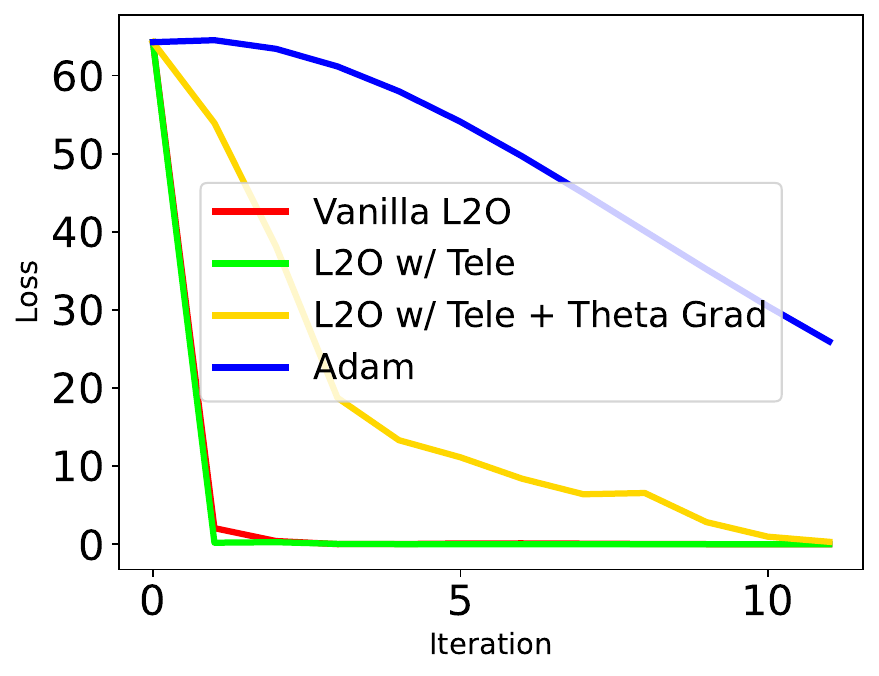}
    \end{minipage}
\caption{Comparison of vanilla L2O with and without learned teleportation for fixed objective Rosenbrock functions.}
\label{fig:rf}
\end{figure}

\begin{figure}[H]
    \centering
    \begin{minipage}{0.5\textwidth}
        \centering
        \includegraphics[width=\textwidth]{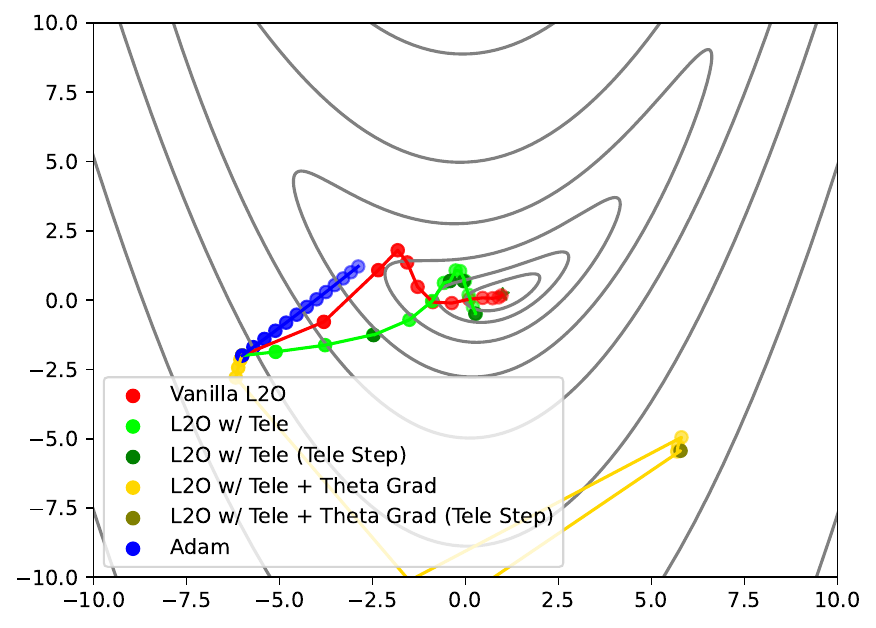}
    \end{minipage}\hfill
    \begin{minipage}{0.5\textwidth}
        \centering
        \includegraphics[width=\textwidth]{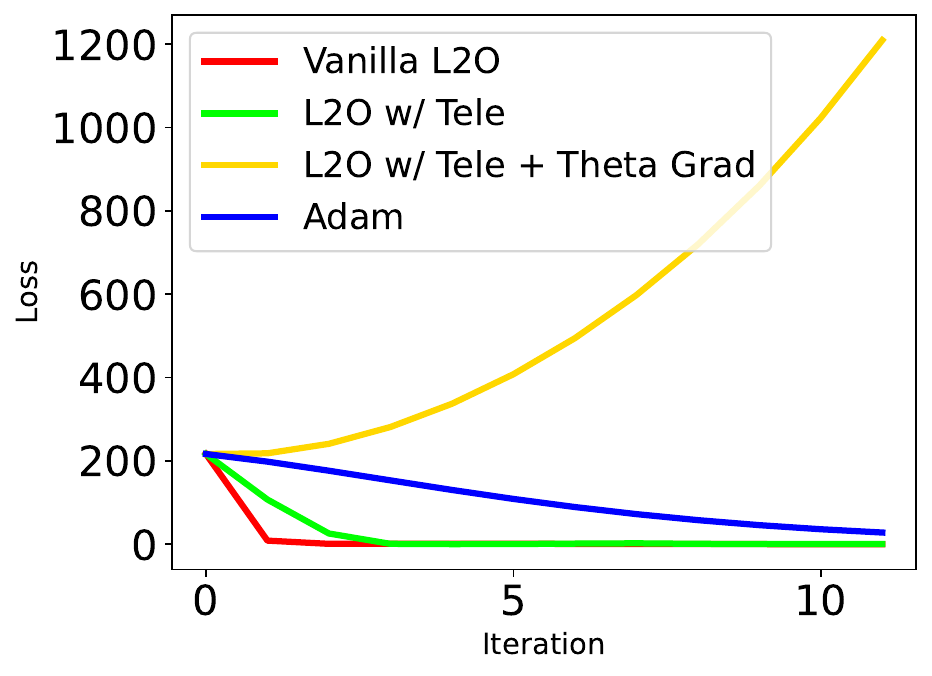}
    \end{minipage}
\caption{Comparison of vanilla L2O with and without learned teleportation for variable objective Rosenbrock functions.}
\label{fig:rv}
\end{figure}


\subsection{Augmenting learning to teleport with momentum}

In this section, we aim to improve on the learning to teleport algorithm proposed in \cite{zhao2023improving}. 
In particular, we augment the algorithm by additional learning a momentum term (Algorithm \ref{alg:momentum}).
We train and test on two-layer LSTMs with hidden dimension 300.  We train the meta-optimizers on multiple trajectories each at an increment of 100 step. The learning rate for meta-optimizers are $10^{-4}$ for $m_1$ and $10^{-3}$ for $m_2$. LSTM learns momentum by passing input as current gradient and outputs momentum coefficient. The velocity is updated with the momentum coefficient.
We evaluate the meta-optimizers on 5 unseen trajectories that were not used during training.

\begin{algorithm}
\caption{Learning to Teleport with LSTM-based Momentum}\label{alg:momentum}
\begin{algorithmic}
\State {\bfseries Input:} L2O optimizer $m_1,m_2$ with initial parameters $\phi_1, \phi_2$, Optimizer $f$, teleportation  $Z$, number of epochs $epoch$.
\For{$t = 1,\dots,epoch$}
    \State $f_{t}, h_{1t} \gets m_1(\nabla_t, h_{1t-1}, \phi_1)$ 
    \State $\beta_t, g_t, h_{2_{t-1}} = m_2(\nabla_t, h_{2_t-1}, \phi_2)$\Comment{LSTM-based Momentum}
    \State $\mathbf{v}_{t} \gets \beta_t \mathbf{v}_{t-1} - \alpha \nabla_t$ \Comment{Momentum update}
     \State ${W}_t \gets {W}_{t-1} + \mathbf{v}_t$ \Comment{Weight update}
    \State $\mathbf{x}_t \gets g_t(\mathbf{x}_{t-1} + \mathbf{v}_{t})$
    \If{ $t \in  Z$}
        \State $\mathbf{x}_t \gets g(\mathbf{x}_{t-1} +\beta _t v_t - \alpha \nabla_t)$
    \EndIf
\EndFor
\end{algorithmic}
\end{algorithm}

\begin{wrapfigure}{R}{0.4\textwidth}
\vskip -12pt
\centering
\includegraphics[width=0.35\textwidth]{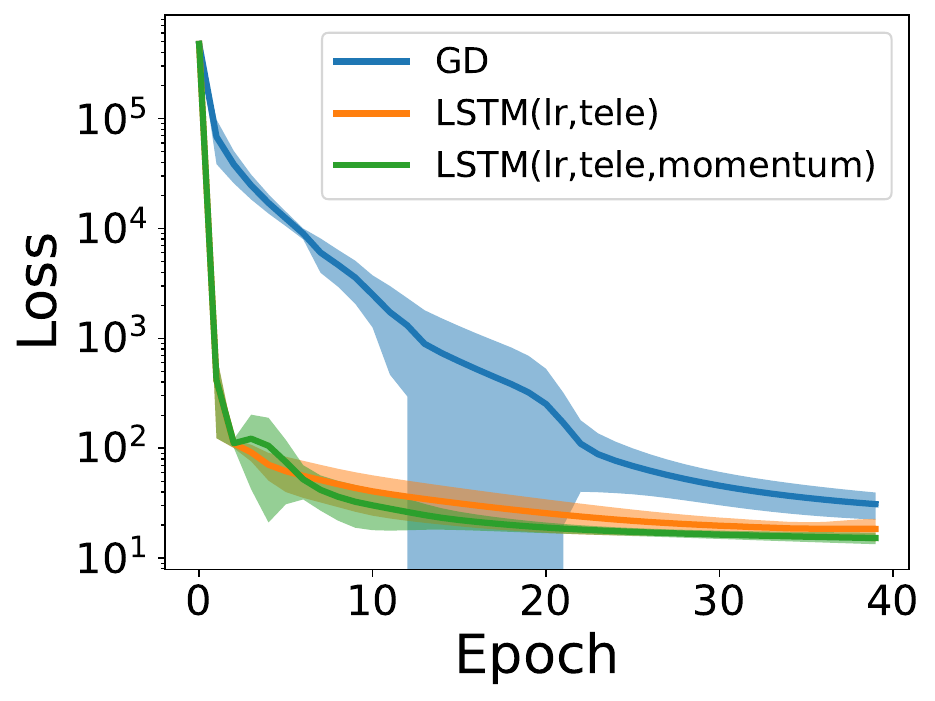}
\caption{Augmenting learning to teleport with learned momentum.}
\label{fig:learningmomentum}
\vskip -20pt
\end{wrapfigure}

We compare with two baselines, vanilla gradient descent (“GD”), and a meta optimizer that learns the group element $g_t$ and the learning rate used to perform gradient descent (“LSTM(lr,tele)”). 
In comparison, our proposed method (“LSTM(lr,tele,momentum)”) learns the momentum coefficient $\beta_t$ and the learning rate used to perform gradient descent.
We continue training until adding more trajectories no longer improves the convergence rate. 
We observe that learning momentum on top of local and symmetry updates improves the meta-learning optimizer (Figure \ref{fig:learningmomentum}), suggesting there is more room for improvement of learning to teleport.


\section{Discussion}

\paragraph{Connections to Reinforcement Learning}



The L2O framework can be naturally framed as a reinforcement learning (RL) problem: the state encodes the iterate and gradient, the action is the update step, the cost is the objective value, and the goal is to learn a policy that minimizes cost over its trajectory. Prior work has explored this RL perspective (e.g., \cite{li2016learning, li2017learning}).
This perspective suggests several avenues for improvement. Future work could refine the cost function, e.g., incorporating distance to the optimum to mitigate harmful teleportation. Additionally, alternative policy architectures or learning algorithms from RL could enhance performance.

\paragraph{Limitations}



Incorporating level set teleportation does not guarantee faster convergence, even in convex settings \citep{mishkin2023level}. Teleporting in the direction of gradient increase can sometimes move iterates further from the optimum or disrupt learned optimization dynamics, leading to overshooting. 
If teleportation is frequently detrimental within a task distribution, the learned optimizer will perform no better than vanilla L2O. However, when the task distribution is well-behaved, teleportation can accelerate convergence. Encouragingly, our experiments suggest that neural network training may fall into this favorable category.


\subsubsection*{Acknowledgments}
This work was supported in part by the U.S. Army Research Office under Army-ECASE award W911NF-07-R-0003-03, the U.S. Department Of Energy, Office of Science, IARPA HAYSTAC Program, and NSF Grants \#2205093, \#2146343, \#2134274, CDC-RFA-FT-23-0069, DARPA AIE FoundSci and DARPA YFA.

\bibliography{iclr2025_conference}
\bibliographystyle{iclr2025_conference}

\appendix
\section*{Appendix}

\section{Proof of Proposition \ref{prop:positive-derivative-convex}}
\label{appendix:positive-derivative-convex}

\begin{lemma}
\label{lemma:quadratic-form-ineq-general}
For vector $\vw \in \mathbb{R}^n$, positive definite matrix $A \in \mathbb{R}^{n \times n}$, and $\alpha, \beta \in \mathbb{Z}$,
\begin{align}
    (\vw^T A^\alpha \vw)^2 \leq (\vw^T A^{\alpha + \beta} \vw) (\vw^T A^{\alpha - \beta} \vw).
\label{eq:quadratic-form-ineq-general}
\end{align}
\end{lemma}

\begin{proof}
Since $A$ is positive definite, there exists orthogonal matrix $P \in \R^{n \times n}$ and diagonal matrix $D \in \R^{n \times n}$ such that $A$ can be decomposed into $A = PDP^{-1}$. Then $A^{-1} = (PDP^{-1})^{-1} = PD^{-1}P^{-1}$, and $A^2 = PDP^{-1}PDP^{-1} = P D^2 P^{-1}$. Similarly $A^\gamma = P D^\gamma P^{-1}$ for all integer $\gamma$.

Let $\vx = P^T\vw \in \R^n$. Substitute $\vw = P \vx$ into the left and right hand side of \eqref{eq:quadratic-form-ineq-general} and let $d_i$ be the $i^{th}$ diagonal element in $D$, we have
\begin{align}
    (\vw^T A^\alpha \vw )^2
    &= \left(\vx^T P^T (P D^\alpha P^{-1}) P \vx\right)^2 \cr
    &= (\vx^T D^\alpha \vx)^2 \cr
    &= \left(\sum_{i=1}^n d_i^\alpha x_i^2\right) \left(\sum_{j=1}^n d_j^\alpha x_j^2\right) \cr
    &= \sum_{i \leq j} (2 d_i^\alpha d_j^\alpha) x_i^2 x_j^2
\label{eq:quadratic-form-ineq-general-lhs}
\end{align}
and 
\begin{align}
    (\vw^T A^{\alpha+\beta} \vw) (\vw^T A^{\alpha-\beta} \vw) 
    &= (\vx^T P^T (P D^{\alpha+\beta} P^{-1}) P \vx) 
    (\vx^T P^T (P D^{\alpha-\beta} P^{-1}) P \vx) \cr 
    &= (\vx^T D^{\alpha+\beta} \vx)(\vx^T D^{\alpha-\beta} \vx) \cr
    &= \left(\sum_{i=1}^n d_i^{\alpha+\beta} x_i^2 \right) \left(\sum_{j=1}^n d_j^{\alpha-\beta} x_j^2 \right) \cr
    &= \sum_{i \leq j} (d_i^{\alpha+\beta} d_j^{\alpha-\beta} + d_i^{\alpha-\beta} d_j^{\alpha+\beta}) x_i^2 x_j^2.
\label{eq:quadratic-form-ineq-general-rhs}
\end{align}

To prove \eqref{eq:quadratic-form-ineq-general-lhs} $\leq$ \eqref{eq:quadratic-form-ineq-general-rhs}, it suffices to show that $2 d_i^\alpha d_j^\alpha \leq d_i^{\alpha+\beta} d_j^{\alpha-\beta} + d_i^{\alpha-\beta} d_j^{\alpha+\beta}$ for all $(i, j)$. Since $d_i$ are the eigenvalues of $A$, all $d_i$'s are positive. By the inequality of arithmetic and geometric means,
\begin{align}
    (d_i d_j^{-1})^\beta + (d_i^{-1}d_j)^\beta \geq 2 \sqrt{(d_i d_j^{-1})^\beta (d_i^{-1}d_j)^\beta} = 2.
\end{align}
Therefore,
\begin{align}
    d_i^{\alpha+\beta} d_j^{\alpha-\beta} + d_i^{\alpha-\beta} d_j^{\alpha+\beta}
    = \left[ (d_i d_j^{-1})^\beta + (d_i^{-1}d_j)^\beta \right] (d_i d_j)^\alpha
    \geq 2 d_i^\alpha d_j^\alpha.
\end{align}
\end{proof}

\begin{manualproposition}{\ref{prop:positive-derivative-convex}}
For convex function $\loss$, the directional derivative of $\left\| \frac{\partial \loss}{\partial \vw} \right\|_2^2$ along the direction of $\vv_{\bot}$ is non-negative. That is,
\begin{align}
    \vv_{\bot} \cdot \frac{\partial}{\partial \vw} \left\| \frac{\partial \loss}{\partial \vw} \right\|_2^2 \geq 0.
\end{align}
\end{manualproposition}

\begin{proof}
Note that $\frac{\partial}{\partial \vw} \left\| \frac{\partial \loss}{\partial \vw} \right\|_2^2 = 2 H \grad \loss$. Then we have
\begin{align}
    \vv_{\bot} \cdot \frac{\partial}{\partial \vw} \left\| \frac{\partial \loss}{\partial \vw} \right\|_2^2
    &= 2 (H \grad \loss)^T \vv_{\bot} \cr
    &= - 2 \grad \loss^T H^T H^{-1} \grad \loss + 2 \frac{\grad \loss^T H^{-1} \grad \loss}{\grad \loss^T \grad \loss}\grad \loss^T H^T \grad \loss \cr 
    &= 2 \left( - \grad \loss^T \grad \loss + \frac{\grad \loss^T H^{-1} \grad \loss}{\grad \loss^T \grad \loss}\grad \loss^T H^T \grad \loss \right).
\end{align}
Setting $\alpha=0$ and $\beta=1$ in Lemma \ref{lemma:quadratic-form-ineq-general} and substitute $\grad\loss, H$ for $\vw, A$, we have
\begin{align}
    - \grad \loss^T \grad \loss + \frac{\grad \loss^T H^{-1} \grad \loss}{\grad \loss^T \grad \loss}\grad \loss^T H^T \grad \loss \geq 0.
\end{align}
Therefore,
\begin{align}
    \vv_{\bot} \cdot \frac{\partial}{\partial \vw} \left\| \frac{\partial \loss}{\partial \vw} \right\|_2^2
    = 2 \left( - \grad \loss^T \grad \loss + \frac{\grad \loss^T H^{-1} \grad \loss}{\grad \loss^T \grad \loss}\grad \loss^T H^T \grad \loss \right)
    \geq 0.
\end{align}
\end{proof}

Note that Lemma \ref{lemma:quadratic-form-ineq-general} does not hold for non-convex functions. For example,
\begin{align}
    \alpha = 0, \beta = 1, 
    \vw = 
    \begin{bmatrix}
        1 \\
        3
    \end{bmatrix}, 
    A = 
    \begin{bmatrix}
        1 & 0\\
        0 & -2
    \end{bmatrix}, 
    A^{-1} = 
    \begin{bmatrix}
        1 & 0\\
        0 & -\frac{1}{2}
    \end{bmatrix}. 
\end{align}
\begin{align}
    (\vw^T \vw)^2 = 100 > (-3.5) \times (-17) = (\vw^T A^{-1} \vw) (\vw^T A \vw).
\end{align}
Consequently, Proposition \ref{prop:positive-derivative-convex} may not hold for non-convex functions.

\section{Derivation omitted from Section \ref{sec:theta_update}.}
\label{appendix:theta_update}

We will derive the closed form version of the update rules. 
Recall that $y_t = g_{\theta_t}(x_t)$. Since $\nabla f(x) = 2Ax = 2 B^T B x$, we have
\begin{align*}
    x_{t+1} = y_t - \alpha \nabla_{y_t} f(y_t) 
    = B^{-1} R_\theta B x_t - \alpha 2BBB^{-1}R_{\theta}Bx
    = (B^{-1} - 2 \alpha B) R_\theta Bx
\end{align*}

To simplify notation, let $C = B^{-1} - \alpha 2B$. Then we have $x_{t+1} = CR_\theta Bx_t$.

We next compute $\nabla_{\theta_t} f(x_{t+1})$ using the chain rule:
\begin{align*}
    \nabla_{\theta_t} f(x_{t+1}) 
    &= \frac{\partial f}{\partial x_{t+1}} \frac{\partial x_{t+1}}{\partial \theta_t} \cr
    &= (2Ax_{t+1})^T (CR_{\theta_t + \frac{\pi}{2}} Bx_t) \cr
    &= (2ACR_\theta Bx_t)^T (CR_{\theta_t + \frac{\pi}{2}} Bx_t) \cr
    &= (2x_t^T B R_{-\theta} CA (CR_{\theta_t + \frac{\pi}{2}} Bx_t)
\end{align*}

Note that $CAC = (I-2\alpha A)^2$ is symmetric. Therefore, there exists an eigendecomposition $CAC = Q\Lambda Q$.
Additionally, $Q$ is the same orthogonal matrix as in the eigendecomposition of $A$, and $\Lambda = 
\begin{bmatrix}
    (1-2\alpha \lambda_1)^2 & 0 \\
    0 & (1-2\alpha \lambda_2)^2
\end{bmatrix}$.
Putting it together, we have
$\nabla_{\theta_t} f(x_{t+1}) 
    = 2(Bx_t)^T \Lambda (R_{\frac{\pi}{2}} Bx_t)$
Plugging this into our update rule, $\theta_{t+1} = \theta_t - 2 \beta (Bx_t)^T \Lambda (R_{\frac{\pi}{2}} Bx_t)$.

The rotation $\theta$ aligns with the principal axes of the level sets of $f(x)$. 
In other words, the update rule pushes the iterate in the direction of rotation towards the shorter axis of the level set ellipse corresponding to the principal eigenvector.

\end{document}